\documentclass[twocolumn]{article} 
\usepackage{times}  
\usepackage{helvet} 
\usepackage{courier}  
\usepackage[hyphens]{url}  
\usepackage{graphicx} 
\usepackage{wrapfig}
\usepackage{tabularx}
\newlength\titlebox 
\usepackage{amsmath}
\usepackage{amssymb}
\usepackage{amsthm}
\usepackage{multicol}
\newtheorem{proposition}{Proposition}

\newtheorem{theorem}{Theorem}
\newtheorem{lemma}{Lemma}
\newtheorem{corollary}{Corollary}

\makeatletter
\newenvironment{tablehere}
  {\def\@captype{table}}
  {}
\makeatother

\title{The Choice Function Framework for Online Policy Improvement}
\author{Murugeswari Issakkimuthu, Alan Fern, Prasad Tadepalli  \\
School of EECS\thanks{E-mail: issakkim, alan.fern, prasad.tadepalli@oregonstate.edu} \\
Oregon State University \\
Corvallis, OR 97331 USA}
\date{}

\begin{document}
\maketitle

\begin{abstract}
There are notable examples of online search improving over hand-coded or learned policies (e.g. AlphaZero) for sequential decision making. It is not clear, however, whether or not policy improvement is guaranteed for many of these approaches, even when given a perfect evaluation function and transition model. Indeed, simple counter examples show that seemingly reasonable online search procedures can hurt performance compared to the original policy. To address this issue, we introduce the choice function framework for analyzing online search procedures for policy improvement. A choice function specifies the actions to be considered at every node of a search tree, with all other actions being pruned. Our main contribution is to give sufficient conditions for stationary and non-stationary choice functions to guarantee that the value achieved by online search is no worse than the original policy. 
In addition, we describe a general parametric class of choice functions that satisfy those conditions and present an illustrative use case of the framework's empirical utility.
\end{abstract}

\section{Introduction}
\label{sec:intro}
For many applications of sequential decision making, it is possible to learn or hand-code a reactive policy for online operation. While such policies are computationally cheap to apply, they will generally be sub-optimal and even highly sub-optimal in some states. 
This motivates using additional computation during online operation to improve upon such base policies. The focus of this paper is on approaches that use online lookahead search for this purpose, which we refer to as Online Search for Policy Improvement (OSPI). 

At each state encountered during operation, OSPI approaches use an environment simulator or model to construct a search tree that includes the base policy's action choices along with a subset of off-policy choices. The action values at the root can then be used to select an action. Well-known examples of OSPI include the policy-rollout algorithm \cite{tesauro1997}, which was first shown to improve Backgammon policies, and AlphaZero \cite{Silver1140}, which improves over an underlying greedy policy via Monte-Carlo Tree Search.

Ideally, due to the additional online computation, we would like an OSPI procedure to yield improved performance compared to the base policy. Perhaps more importantly, we would at least like to guarantee that an OSPI procedure is ``safe" in the sense that it does not perform worse than just using the base policy. For example, the policy rollout algorithm is guaranteed to be safe in this sense. However, as we show in Section \ref{sec:setup}, many OSPI procedures are not safe, even when 1) using a perfect model, 2) using the exact policy value function for leaf evaluation, and 3) the policy's actions are expanded at each tree node.    

Our primary goal is to derive safety conditions for OSPI. For this purpose, we introduce the choice-function framework for analyzing OSPI procedures. The key idea is to notice that OSPI procedures primarily differ in their choice of which off-policy actions to expand. Thus, each procedure can be characterized by a choice function, which specifies the actions to consider at each node of the tree. Thus, we can characterize properties of an OSPI procedure, such as safety, via properties of the corresponding choice function. 

Our main contribution is to give sufficient conditions on choice functions that guarantee safeness. This is done for both stationary and non-stationary choice functions. In addition, we describe a parametric class of safe choice functions, that captures a number of existing approaches. This allows for hyper-parameter search over a safe space of OSPI procedures in order to optimize online performance. Using this class we provide illustrative empirical results that demonstrate the practical potential of the framework.  

\section{Related Work}
\label{sec:related}
An early approach for OSPI is the  policy-rollout algorithm \cite{Bertsekas:1996:NP:560669,tesauro1997}, 
which has been shown to significantly improve policies in a variety of applications, e.g. Backgammon \cite{tesauro1997}, combinatorial optimization \cite{bertsekas1997combinatorial} and stochastic scheduling \cite{bertsekas1999scheduling}. Nested rollout \cite{yan2005solitaire,cazenave2009nested} allows for leveraging additional computation time to further improve a policy by approximating multiple steps of policy iteration. Policy Switching \cite{Chang:2004:PRO:990742.990814} allows rolling out multiple policies instead of just one and improves over all the base policies.

Monte-Carlo Tree Search (MCTS) has commonly used policies as a form of knowledge to guide and prune the search, often as part of the rollout policy applied at the leaves \cite{browne2012survey}.  Recent, high-profile examples include AlphaGo and AlphaZero \cite{silver2016alphago,silver2017gozero,Silver1140}, which combine a learned base policy and value function to guide MCTS. One view of AlphaZero's search approach is as OSPI, where the search aims to improve over the learned greedy base policy. Indeed, the basis for learning is to use search to generate training data from a (hopefully) improved policy.  A related approach \cite{pinto2017} uses a learned policy to prune actions from consideration at each tree node that are not highly ranked by the policy. Another example of combining MCTS with policies \cite{nguyen2014bootstrapping} allows the base policy to be treated as a temporally extended action at each node in the search tree. 

The idea of searching around a base policy has also been considered in the area of deterministic heuristic search. Limited discrepancy search (LDS) \cite{harvey1995limited} uses a heuristic to define a greedy policy for guiding search. LDS generates all paths in the search tree that disagree with at most $K$ choices of the base policy and returns the best solution uncovered by the search. LDS has been used effectively in a variety of search problems ranging from standard benchmarks to structured prediction \cite{doppa2014structured} and non-deterministic AND/OR search graphs \cite{larrosa2016limited}.

\section{Problem Setup}
\label{sec:setup}
We formulate sequential decision making in the formalism of Markov Decision Processes (MDPs). An MDP is a $4$-tuple $\langle S, A, P, R \rangle$, where $S$ is a finite set of states, $A$ is a finite set of actions, $P: S \times A \times S \rightarrow [0,1]$ is the state-transition function and $R: S \times A \rightarrow \mathbb{R}$ is the reward function. $P_{ss'}(a)$ denotes the probability of reaching state $s'$ from state $s$ by taking action $a$ and $R(s,a)$ denotes the immediate reward for taking action $a$ in state $s$. A policy, $\pi : S \rightarrow A$, is a mapping from states to actions with value function $V^{\pi}$. We focus on the discounted infinite-horizon setting with discount factor $\gamma$. Offline computation of an optimal policy can be done in polynomial time in the size of the state space, which is impractical for typical applications. In such cases, online search is a practical alternative to offline solutions. \\

{\bf Online Search for MDPs.} In online search, actions need to be selected only for the states actually encountered in the environment. At each decision point, online search typically constructs a finite-horizon search tree rooted at the current state.  A search tree alternates between layers of state and action nodes. The leaves of a search tree are  state nodes and are often evaluated via a state evaluation function. The tree is used to estimate action values at the root and the action with the highest estimate is executed in the environment.
  
When a model of the MDP is available, an expectimax tree can be built that assigns exact probabilities to child states of actions. For large enough search depths or accurate leaf evaluations, near-optimal actions can be selected. In some applications, only a simulator of the MDP is available, which allows for sampling state transitions and rewards. Monte-Carlo sampling can then be used to construct an approximation to the exact tree by sampling a number of child states for each action node. The Sparse Sampling algorithm \cite{Kearns:2002:SSA:599616.599698} follows this approach and guarantees near optimal action selection in time independent of the state-space size. Monte-Carlo Tree Search (MCTS) algorithms \cite{browne2012survey} also use simulators for online search, typically producing non-uniform depth trees. \\
\begin{figure}[t]
\begin{center}
    \includegraphics[width=0.5\textwidth, height=0.19\textheight]{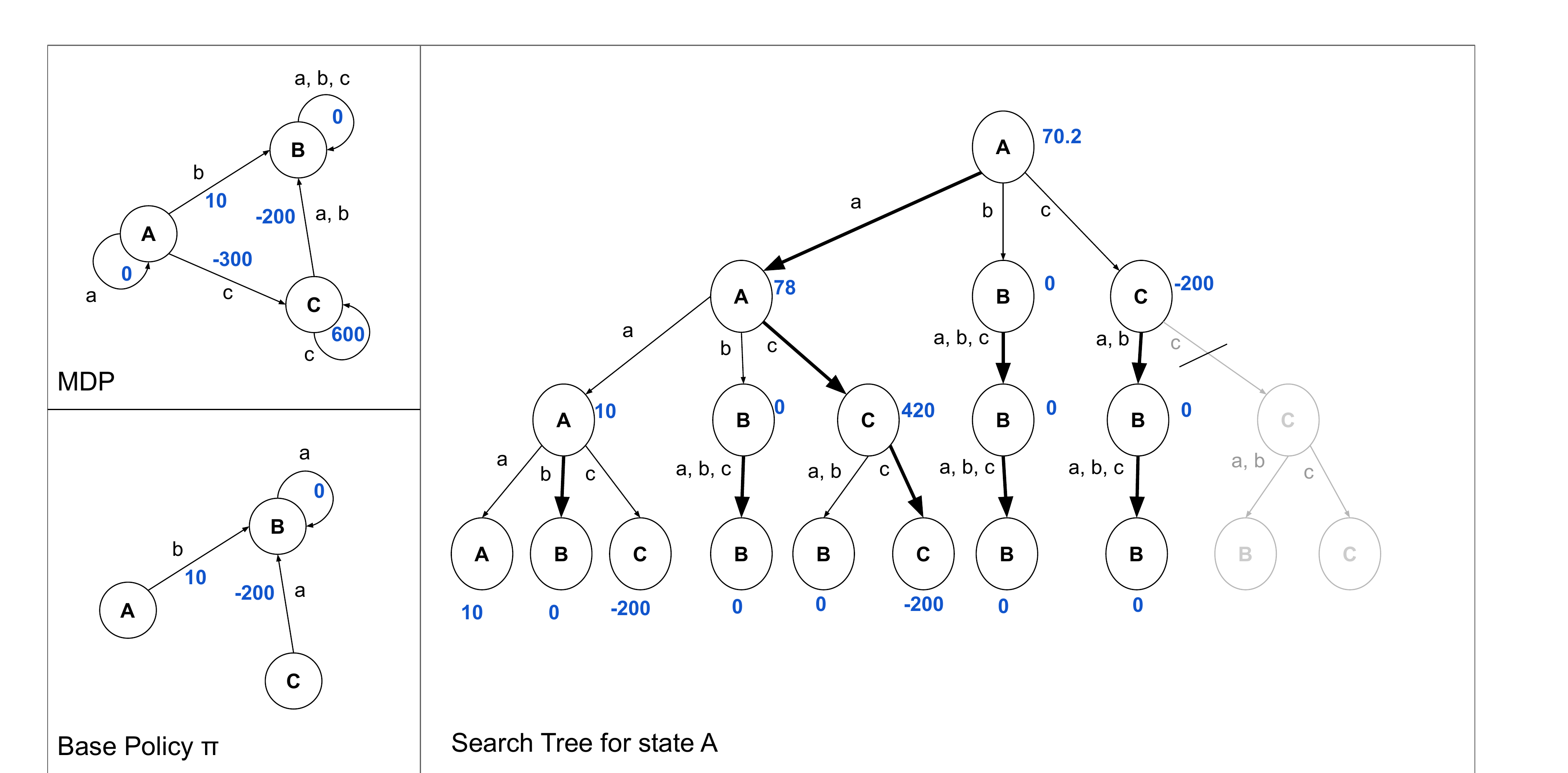}
    \vspace{-1em}
    \caption{Search tree from a search procedure for shown MDP. Each leaf is assigned the value of the shown base policy and state nodes include the action of the policy. The grayed out part of the tree is the part not expanded by the search. 
    The values shown for nodes have  been computed using a discount factor of 0.9. 
    The text describes how the choice of the tree (action $a$) is not $\pi$-safe due to the pruning.}
    \label{fig:1}
    \vspace{-0.5em}
\end{center}
\end{figure}

{\bf Online Search for Policy Improvement.}  In practice, online computational constraints can limit the search tree size, which can lead to poor performance of online search. To address this issue, it is common to learn or provide different types of prior knowledge into the search process. For example, the search depth can be reduced by utilizing higher-quality leaf evaluation functions, or the search breadth can be reduced via action pruning functions.

While such knowledge sources can reduce computational cost, there are typically no guarantees on the value achieved by the online search procedure. Most theoretical results aim to guarantee near optimal performance (e.g. \cite{Kearns:2002:SSA:599616.599698}), but require impractical computational costs. Rather, it is desirable to develop approaches that support performance guarantees within practical computational limits. This motivates the framework of OSPI.

OSPI aims for performance guarantees relative to a base policy $\pi$. The policy may be learned or hand-coded, but is assumed to be computationally cheap to apply. While $\pi$ could be directly used for online action selection, this may not fully use the computational resources. OSPI aims to leverage those resources to improve over $\pi$ via online search to explore the decision space around $\pi$. Formally, an OSPI procedure is \emph{$\pi$-safe} when its online performance is guaranteed to be as good or better than that of $\pi$. That is, if $\pi'$ is the online policy computed by an OSPI procedure, then the procedure is $\pi$-safe if for all states $s$, $V^{\pi'}(s)\geq V^{\pi}(s)$. While safety is a less powerful guarantee compared to near optimality, in practice, it is more attainable and still useful. 

It can be difficult to determine, in general, whether a given OSPI procedure is $\pi$-safe. For example, one might expect that an OSPI procedure that considers the actions of $\pi$ at every tree node and uses $V^{\pi}$ for leaf evaluation might be $\pi$-safe when combined with a perfect environment model. However, this is not the case as the following counter-example shows. The next section develops a framework for assessing the safety of OSPI procedures.

{\bf Counter-example. } Figure \ref{fig:1} shows a deterministic MDP, a base policy $\pi$, and the search tree constructed for state $A$ by an unspecified deterministic search procedure, which produces the same tree each time $A$ is encountered. The tree respects the exact MDP model and the leaf evaluation function is exactly $V^{\pi}$. Further, each node of the search tree includes the action corresponding to the choice of $\pi$. 

The action choice of the search procedure (i.e. the highest valued root actions) will be denoted by the policy $\pi'$. Given the tree properties relative to $\pi$, we might expect that the value of $\pi'$ would be at least as good as $\pi$. For state $A$, the base policy selects $\pi(A)=b$ and the corresponding value of state $A$ under $\pi$ is $V^{\pi}(A) = 10$. However, the online search suggests the alternative action $\pi'(A)=a$, which results in a lower value of $V^{\pi'}(A)=0$. Thus, the online search procedure is not $\pi$-safe, at least for state $A$.

To understand the failure to be $\pi$-safe at state $A$, consider using the search tree to make a decision at state $A$ at time step $t$. The reason action $a$ looks best is that the state-action sequence $A\rightarrow a \rightarrow A \rightarrow c \rightarrow C \rightarrow c \rightarrow C$ achieves a high value due to the 600 reward of the final transition. However, after actually taking action $a$ and ending up in state $A$ again at time step $t+1$ the tree does not include the promising path of $A\rightarrow c \rightarrow C \rightarrow c \rightarrow C$, due to pruning of the lower levels of the search tree. Thus, at time step $t+1$ the search procedure does not recognize the value of taking action $c$ in $A$ and takes $a$ again. This is just one of several failure-mode types of OSPI procedures, even when their trees satisfy the assumptions of this example relative to $\pi$.

\section{The Choice Function Framework}
\label{sec:framework}
Search trees encode the future trajectories to be considered when 
evaluating actions at a state. 
Search algorithms vary in how they expand the paths, which results in different  search trees and hence different action values. Thus, one way to characterize online search approaches is by describing the trees they construct. In our framework, this is done using choice functions, which allows for properties, such as safety of search, to be analyzed via choice-function properties.

\subsection{Choice Functions for General Online Search}
Search trees have two sources of branching: 1) action branching and 2) state branching. Choice functions describe the action branching by specifying the subset of possible action choices to be considered at each state node. Leaf nodes are assigned the empty set of choices. For state branching, we assume an exact MDP model so that all non-zero probability child states of an action node are included in the tree. When a model is not available, but a simulator is,  sparse sampling can be used to approximate the dynamics.

State and action nodes in a tree are identified by paths that list the alternating sequence of states and actions starting at the root state. The path of a state node labeled with state $s$ will be denoted by $p;s$ where $p$ is the path starting at the root leading to the parent action of the state node. Action nodes will often similarly be denoted by $p;s;a$, where $p;s$ designates the parent state node. The length of any path denoted by $|p|$ is the number of actions that it contains. Thus, a path corresponding to a single state $s$ has length zero. The set of all state paths is denoted by $\mathcal{SP}$ and a \emph{choice function} is a mapping $\psi: \mathcal{SP} \rightarrow 2^A$ from state paths to action subsets. 

In order to define the trees associated with a choice function, several definitions are needed. A path is \emph{$\psi$-satisfying} if all of its actions are ``allowed" by $\psi$. That is, for each prefix $p';s';a'$ of the path, we have $a' \in \psi(p';s')$. A state path $p;s$ is a \emph{leaf path} of $\psi$ if it is $\psi$-satisfying and $\psi(p;s)=\emptyset$. Note that a leaf path of $\psi$ cannot be extended to a $\psi$-satisfying path. A choice function $\psi$ is \emph{finite horizon} if there is a finite upper bound on the length of any $\psi$-satisfying state path. For finite-horizon $\psi$, the \emph{horizon} $H(\psi)$ is the maximum length of any $\psi$-satisfying path, or equivalently, of any leaf path.

Given a current, or root state, $s_0$, the tree corresponding to $\psi$, denoted $T^{\psi}(s_0)$, is the tree containing exactly the $\psi$-satisfying state paths that begin with $s_0$. Thus, the leaves of $T^{\psi}(s_0)$ correspond to leaf paths of $\psi$. The tree will be finite when $\psi$ is finite horizon, with $H(\psi)$ bounding the depth of any leaf. In this paper, we will restrict attention to finite-horizon choice functions and hence finite trees.

To use the tree $T^{\psi}(s_0)$ for action selection at state $s_0$, it is necessary to specify a \emph{leaf evaluation function} $u$, which is a function of states $u: S \rightarrow \mathbb{R}$. Often $u$ will be a learned or hand-coded function that provides an estimate of a state's optimal value or value under a policy. Alternatively, $u$ may be uninformative and return a constant value. Together, a choice function $\psi$ and leaf evaluation function $u$ allow us to define the value  of each state node $p;s$ in $T^{\psi}(s_0)$, denoted $V^{\psi}_u(p;s)$, and each action node $p;s;a$, denoted $Q^{\psi}_u(p;s;a)$.
\begin{small}
\begin{eqnarray*}
V^\psi_u(p;s) &=& \left \{
\begin{array}{lr}
u(s), & \psi(p;s) = \emptyset \\
\displaystyle \max_{a \in  \psi(p;s)} Q^{\psi}_u(p;s;a), & \text{\emph{ otherwise}} \\
\end{array}
\right. \\
Q^{\psi}_u(p;s;a) &=& R(s,a) + \gamma \displaystyle \sum_{s' \in S} P_{ss'}(a)\cdot V^\psi_u(p;s;a;s')
\end{eqnarray*}
\end{small}

\noindent The online-search action policy, denoted $\Pi^{\psi}_u$, returns a maximum valued action at state $s$ allowed by $\psi$, with ties broken arbitrarily: 
$\Pi^\psi_u(s) = \arg \displaystyle \max_{a \in \psi(s)} Q^{\psi}_u(s;a)$.

\subsection{Choice Functions for OSPI}
Given a base policy $\pi$ we would like to define choice functions and corresponding leaf-evaluation functions that result in (approximately) $\pi$-safe OSPI procedures. That is, we would like to guarantee that $\Pi^{\psi}_u$ is $\pi$-safe. Section \ref{sec:performanceguarantees} develops sufficient conditions on choice functions to give such a guarantee. First, however, we provide examples of choice functions for several existing OSPI procedures whose safety will later be assessed according to the conditions.\\

\textbf{Policy Rollout.} This simple OSPI procedure \cite{tesauro1997} returns the action at state $s$ that maximizes a Monte-Carlo estimate of $Q^\pi(s,a)$. This estimate can be viewed as evaluating a tree that considers all actions at the root $s$ and then only contains the actions of $\pi$ thereafter until some horizon $H$. Policy rollout can thus be characterized by the following choice function. 

\begin{small}
\begin{flalign*}
\psi_{\text{ro}}(p;s) = \left \{
\begin{array}{lr}
A, & |p|  = 0 \\
\pi(s), & 0< |p| < H \\
\emptyset, & \text{otherwise} \\
\end{array}
\right.
\end{flalign*}
\end{small}
Policy rollout can be proven to be $\pi$-safe as it corresponds to the policy improvement step of the policy iteration algorithm. Our $\pi$-safe conditions will imply this for $\psi_{\text{ro}}$.\\

\textbf{Limited Discrepancy Search (LDS).} This procedure was originally introduced for deterministic offline search problems \cite{harvey1995limited}. LDS searches around $\pi$ by limiting the number of off-policy actions (discrepancies) along every root-to-leaf path to $K$ up to some maximum horizon $H$. The idea was later extended to offline non-deterministic AND/OR tree/graph search \cite{larrosa2016limited} using a similar limit on discrepancies. LDS for MDPs is easily captured via the following choice function, where $\#[p \not= \pi]$ is the number of off-policy actions in path $p$. 
\begin{small}
\begin{flalign*}
\psi_{lds}(p;s) = \left \{
\begin{array}{lr}
A, & |p| < H \text{ and } \#[p \not= \pi] < K \\
\pi(s), & |p| < H \text{ and } \#[p \not= \pi] = K \\
\emptyset, & |p|=H \\
\end{array}
\right.
\end{flalign*}
\end{small}
Our conditions will imply that $\psi_{\text{lds}}$ is $\pi$-safe.\\

\textbf{Pruned Online Search with Learned Policies.} RL algorithms typically learn policies that select actions by maximizing an action ranking function, such as a Q-function or probability distribution over actions. Such action rankings can be used for action pruning in online tree search. Let $q(p;s, a)$ be the learned action ranking function, which may depend on the full path $p;s$ (e.g. when $q$ is a recurrent neural network) or depend only on $s$. A simple pruning approach, e.g. as studied in \cite{pinto2017}, allows only the set of top $k$ actions at each search node, denoted $\text{TOP}_{q,k}(p;s)$, as captured by the following choice function.  %
\begin{small}
\begin{flalign*}
\psi_{q,k}(p;s) = \left \{
\begin{array}{lr}
\text{TOP}_{q,k}(p;s), & |p| < H  \\
\emptyset, & |p|=H \\
\end{array}
\right.
\end{flalign*}
\end{small}
As discussed in Section \ref{sec:LDCF}, our results will help clarify conditions on $q$ that ensure safety. 
\section{Performance Guarantee}
\label{sec:performanceguarantees}
Our goal is to identify properties of a choice function $\psi$ that guarantee $\Pi^\psi_u$ is approximately $\pi$-safe. That is, we seek to bound $V^\pi(s) - V^{\Pi^\psi_u}(s)$. A natural property to suggest is that $\psi$ be consistent with $\pi$. A choice function $\psi$ is $\pi$-consistent if for each state path $p;s$, $\pi(s) \in \psi(p;s)$. Our counter example (Section \ref{sec:setup}), however, is based on a $\pi$-consistent choice function, since all tree nodes include $\pi$. Thus, $\pi$-consistency of $\psi$ is not sufficient for $\pi$-safety, requiring the introduction of additional concepts and notation. 

We will often treat value functions as vectors, indexed by states, with arithmetic and comparison operators being applied element-wise. The max-norm of a vector $\lVert V \rVert_\infty$ returns the maximum absolute value of the elements. The \emph{min-horizon} of $\psi$, denoted $h(\psi)$, is the minimum depth of any leaf node in $T^{\psi}$. Given a path $p;s$ we let $\lrcorner p;s$ denote the path obtained by removing the first state action pair of $p;s$. We say that $\psi$ is \emph{ monotonic} if $\psi(p;s) \subseteq \psi(\lrcorner p;s)$ for all $p;s \in \mathcal{SP}$. We can now give our main result. \\
\begin{theorem}
\label{thm:stationary}
For any MDP, discount factor $\gamma$, and policy $\pi$, if $\psi$ is $\pi$-consistent and monotonic and $\lVert u - V^\pi \rVert_\infty \leq \epsilon$, then for $\pi'=\Pi^{\psi}_u$,
\begin{small}
\begin{flalign*}
V^\pi - V^{\pi'} \leq \dfrac{2 \epsilon \gamma^{h(\psi)}}{1 - \gamma}. \hspace{5em}
\end{flalign*}
\end{small}
\end{theorem}
\noindent This bound implies that in the ideal case when $u=V^{\pi}$, monotonicity and $\pi$-consistency together are sufficient for safety. It also shows that the impact of inaccuracy in $u$ with respect to $V^{\pi}$ decreases exponentially with the min-horizon due to discounting of future returns. 

From the theorem we get an immediate corollary that applies to the set of all policies 
$\psi$ is consistent with,
denoted $\mathcal{C}_{\psi}$, where $\epsilon(u,\pi)= \lVert u - V^\pi \rVert_\infty$. 
\begin{corollary}
\label{corr:lowerbound}
For any MDP, discount factor $\gamma$, choice function $\psi$, and leaf evaluation function $u$, the policy $\pi'=\Pi^{\psi}_u$ satisfies \vspace{-1em}
\begin{small}
\begin{align*}
    V^{\pi'} \geq \displaystyle \max_{\pi \in \mathcal{C}_{\psi}} \bigg [V^\pi - \dfrac{2 \epsilon(u,\pi) \gamma^{h(\psi)}}{1 - \gamma} \bigg ]. \hspace{4em}
\end{align*}
\end{small}
\end{corollary}
\noindent This implies that for a large min-horizon the online policy is guaranteed to be safe with respect to the best policy that $\psi$ is consistent with. In general, this shows the performance trade-off between larger min-horizons (i.e., minimum search depth) and the closeness of $u$ to a good policy. \\
\subsection{Proof of Theorem \ref{thm:stationary}}
\label{sec:analysis}
All proofs not in the main text are in the appendix. 
The high-level idea of our proof is inspired by the analysis of offline multi-step policy improvement \cite{Bertsekas:1996:NP:560669}. This procedure starts with the value function $V_0 = V^{\pi}$ and then performs $m$ applications of the \emph{Bellman Backup} operator $B$ to get a sequence of value functions $V_i = B[V_{i-1}]$, where $B$ is defined as follows. 
$$B[V](s) = \max_{a\in A} R(s, a) + \gamma \displaystyle \sum_{s' \in S} P_{ss'}(a) \text{ } V(s')$$
The greedy policy $\pi_m$ with respect to the final value function $V_m$ is then returned as the improved policy over $\pi$. The monotonicity of $B$ can be used to guarantee that $V^{\pi_m} \geq V^{\pi}$.  

An OSPI procedure for computing $\pi_m(s)$ is to evaluate a depth $m+1$ tree with root $s$ using $V^{\pi}$ for leaf values. This computes the greedy action with respect to $V_m$, but without synchronously updating all states from the bottom up. Thus, the offline guarantee carries over to OSPI. OSPI with choice functions can be viewed similarly but with backups restricted to actions allowed by the choice function. Our analysis below generalizes these ideas path-sensitive choice functions and approximate leaf values. 

To prove the main result we start by introducing a number of lemmas. 
It will be useful to introduce the \emph{policy-restricted Bellman Backup} operator $B_{\pi}[V]$, which restricts backups to only consider the actions of $\pi$. 
\begin{flalign*}
B_\pi[V](s) = R(s, \pi(s)) + \gamma \displaystyle \sum_{s' \in S} P_{ss'}(\pi(s)) \text{ } V(s').
\end{flalign*}
Lemma \ref{lemma:ldiff} gives a lower bound on $V^\pi$ in terms of a value vector $V$ and how much $B_{\pi}$ decreases the value of $V$.

\begin{lemma}
\label{lemma:ldiff}
For any policy $\pi$ and value vector $V$, if \\ $V - B_{\pi}[V] \leq \delta$, then $V - V^{\pi} \leq \dfrac{\delta}{1 - \gamma}$. 
\end{lemma}

Next, Lemma \ref{lemma:monotonicity} generalizes the conditions that guarantee policy improvement in the offline multi-step lookahead policy improvement procedure to OSPI with choice functions. Proposition \ref{prop:chaining} follows from lemma \ref{lemma:monotonicity}. 

\begin{lemma}
\label{lemma:monotonicity}
If a stationary choice function $\psi$ is $\pi$-consistent and monotonic and $u = V^\pi$ then for any path $p;s$ such that $1 \leq |p;s| \leq H(\psi)$, $V^\psi_u(\lrcorner p;s) \geq V^\psi_u(p;s)$.
\end{lemma}

\begin{proposition}
\label{prop:chaining}
If a stationary choice function $\psi$ is $\pi$-consistent and monotonic and $u = V^\pi$, then $V^\psi_{u}(s) \geq V^\pi(s)$.\\
\end{proposition}

Lemma \ref{lemma:contraction} below bounds the values of paths of length less than or equal to $h(\psi)$ for two different leaf evaluation functions $u$ and $u'$ satisfying $\lVert u - u' \rVert_\infty \leq \epsilon$. This will be useful for quantifying the impact of the leaf evaluation function being an approximation to the base policy value function. 

\begin{lemma}
\label{lemma:contraction}
If $\psi$ is a stationary choice function and $\lVert u - u' \rVert_\infty \leq \epsilon$ then for any path $p;s$ with $|p;s| \leq h(\psi)$, $\left | V^{\psi}_u(p;s) - V^{\psi}_{u'}(p;s) \right | \leq \epsilon \gamma^{h(\psi) - |p;s|}$.
\end{lemma} 

For the following we use the notation $V_{u,k}^\psi$ to denote the vector consisting of the elements of $V^\psi_u$ for $|p;s| = k$. In particular, $V_{u, 0}^\psi$ is the vector of the values of all root nodes, i.e., $|p;s| = 0$. Lemma \ref{lemma:qadvantage} and proposition \ref{prop:policyvalue} below are key results that combine to bound the difference between $V_{u, 0}^\psi$ and the value of $\Pi^\psi_u$.

\begin{lemma}
\label{lemma:qadvantage}
If a stationary choice function $\psi$ is $\pi$-consistent and monotonic  and $\lVert u - V^\pi \rVert_\infty \leq \epsilon_\pi$, then for $\pi' = \Pi^\psi_u$,
$V_{u, 0}^\psi - B_{\pi'}[V_{u, 0}^\psi] \leq \epsilon_\pi \gamma^{h(\psi)} (1 + \gamma)$.
\end{lemma}

\begin{proposition}
\label{prop:policyvalue}
If a stationary choice function $\psi$ is $\pi$-consistent and monotonic  and $\lVert u - V^\pi \rVert_\infty \leq \epsilon_\pi$, then for $\pi' = \Pi^\psi_u$,
\begin{flalign*}
V_{u, 0}^\psi - V^{\pi'} \leq \dfrac{\epsilon_\pi \gamma^{h(\psi)} (1 + \gamma)}{1 - \gamma}. \hspace{6em}
\end{flalign*}
\begin{proof}
Directly combine lemmas \ref{lemma:qadvantage} and \ref{lemma:ldiff}.
\end{proof}
\end{proposition}

\noindent Using the above lemmas we can now prove the main result.
\setcounter{theorem}{0}
\begin{theorem}
\label{thm:stationary2}
If a stationary choice function $\psi$ is $\pi$-consistent and monotonic and $\lVert u - V^\pi \rVert_\infty = \epsilon_\pi$, then for $\pi' = \Pi^\psi_u$,
\begin{small}
\begin{flalign*}
V^\pi - V^{\pi'} \leq \dfrac{2 \epsilon_\pi \gamma^{h(\psi)}}{1 - \gamma}. \hspace{10em}
\end{flalign*}
\end{small}
\end{theorem}
\begin{proof}
Let $\bar{u} = V^\pi$ to simplify notation.
\begin{small}
\begin{flalign*}
V^\pi(s) - V^{\pi'}(s) &= V^\pi(s) - V_{u, 0}^\psi(s) +  V_{u, 0}^\psi(s) -  V^{\pi'}(s), \hspace{3em} \\ 
& \hspace{-5em} \leq V^\pi(s) - (V_{\bar{u}, 0}^\psi(s) - \epsilon_\pi \gamma^{h(\psi)}) + V_{u, 0}^\psi(s) -  V^{\pi'}(s), \\
& \hspace{14.2em} \text{\emph{ by lemma \ref{lemma:contraction}}} \\
& \hspace{-5em} \leq \epsilon_\pi \gamma^{h(\psi)} + V_{u, 0}^\psi(s) - V^{\pi'}(s),  \hspace{0.7em} \text{\emph{ since }} V_{\bar{u}, 0}^\psi(s) \geq V^\pi(s)  \\
& \hspace{12.2em} \text{\emph{ by proposition \ref{prop:chaining}}} \\
& \hspace{-5em} \leq \epsilon_\pi \gamma^{h(\psi)} + \dfrac{\epsilon_\pi \gamma^{h(\psi)} (1 + \gamma)}{1 - \gamma}, \hspace{4em} \text{\emph{ by proposition \ref{prop:policyvalue}}} \\
& \hspace{-5em} = \dfrac{2 \epsilon_\pi \gamma^{h(\psi)}}{1 - \gamma}.
\end{flalign*}
\end{small}
\end{proof}

\subsection{Non-Stationary Choice Functions}
\label{sec:analysisextensions}

We have assumed that choice functions are stationary, i.e., the same choice function is used across online decision steps. Some OSPI approaches, however, correspond to a non-stationary choice function which vary across steps. For example, some search algorithms use a sub-tree produced at time step $t$ as a starting point for search at time step $t+1$ or randomized OSPI approaches are non-stationary due to different random seeds across steps. Here, we extend our analysis to the non-stationary case. 

A non-stationary choice function $\Psi = (\psi_1, \psi_2, \psi_3, \hdots)$ is a sequence of time-step indexed stationary choice functions $\psi_t$. To relate two different stationary choice functions $\psi$ and $\psi'$ we say that $\psi$ subsumes $\psi'$, denoted $\psi \supseteq \psi'$, if for every path $p;s$, $\psi(p;s) \supseteq \psi'(p;s)$. We can extend the bound in Theorem \ref{thm:stationary2} to a non-stationary choice function $\Psi$ when each $\psi_t$ satisfies the conditions of that theorem, each $\psi_t$ has the same set of leaf paths, and $\psi_{t+1} \supseteq \psi_{t}$ for each time-step $t$. 

\begin{theorem}
\label{thm:t2}
Let $\Psi = (\psi_1, \psi_2, \hdots)$ be a non-stationary choice function such that each component choice function $\psi_t$ is monotonic and $\pi$-consistent and $\lVert u - V^\pi \rVert_\infty = \epsilon_\pi$. If all $\psi_t$ have the same set of leaf paths and for each time-step $t$, $\psi_{t+1} \supseteq \psi_t$, then for $\pi' = \Pi^\psi_u$ and all $s \in S$,
\begin{small}
\begin{flalign*}
V^\pi(s) - V^{\pi'}(s) \leq \dfrac{2 \epsilon_\pi \gamma^{h(\psi_1)}}{1 - \gamma}. \hspace{6em}
\end{flalign*}
\end{small}
\end{theorem}
The proof is in the appendix. 
This result has implications on the design of OSPI procedures that correspond to non-stationary choice functions. For example, many MCTS-based approaches, such as that used by AlphaZero, do not appear to correspond to non-stationary choice functions that satisfy these conditions. This does not mean that they will not perform well in a particular application, but suggests that for some applications they can have fundamental issues that degrade the performance of a base policy, even ignoring inaccuracies due to Monte-Carlo sampling. One way to adjust some of these and other algorithms to achieve the subsumption property is to build upon the relevant subtrees constructed at time $t$ at time step $t+1$. The practical impact of such a change is an empirical question worth investigating in future work.

\section{Limited Discrepancy Choice Functions}
\label{sec:LDCF}
We do not expect a single type of choice function to perform best across all problems. Rather, in practice, the selection of a choice function can be similar to the selection of other application-dependent hyperparameters. This motivates defining parametric families of choice functions that span different trade-offs between decision time and quality. In particular, given an application's decision-time constraints, offline optimization can be used to select a high-performing choice function that satisfies the constraints.

To support this vision, we introduce the parametric family of \emph{limited discrepancy choice functions (LDCFs)}. We show that all LDCFs are monotonic and $\pi$-consistent for any $\pi$ and hence satisfy our safety conditions. We then analyze how the parameters of the LDCF family relate to the computational complexity of online action selection. Finally we relate LDCFs to previously introduced examples.  \\

{\bf LDCF Definition and Safety.}
An LDCF $\psi$ defines a uniform-horizon tree, which limits the discrepancies w.r.t. a base policy along root-to-leaf paths by their number and depth. 
In cases where a base policy only makes occasional errors the discrepancy limit makes intuitive sense. Indeed, search can improve the policy by ``correcting" the rare errors along paths by introducing discrepancies. This suggests that there can be a computational advantage to bounding search by discrepancies in applications of OSPI to learned policies that already perform well but can still be improved.\\

A discrepancy w.r.t. $\pi$ in path $p$ is a consecutive state-action pair $(s,a)$ in $p$ such that $a \neq \pi(s)$.
LDCFs are parameterized by the tuple $(\pi, H, K, D, \Delta)$, where $\pi$ is the base policy, $H$ is the uniform horizon bound, $K \leq H$ is a bound on the number of discrepancies in a path, and $D < H$ is a bound on the maximum depth that a discrepancy may appear in a path. Finally, the \emph{discrepancy proposal function} $\Delta : S \times \{0,\ldots, D\} \rightarrow 2^A$ maps pairs of states and depths (excluding the leaf depth) to action subsets. Intuitively $\Delta$ returns the discrepancies that can be considered at a state node $p;s$ at depth $|p|$ which has not yet reached the discrepancy limit imposed by $K$ and $D$.
We allow $\Delta$ to depend on depth, since it is often useful to allow for more discrepancies at shallower depths. Given parameters $\theta=(\pi, H, K, D, \Delta)$ the corresponding LDCF is defined as follows.\\
\begin{small}
\begin{flalign*}
\psi_{\theta}(p;s) = \left \{
\begin{array}{lr}
\Delta(s,|p|)\cup \{\pi(s)\}, & \hspace{-2em} |p| \leq D \text{  and  } \#[p \not= \pi] < K \\
\pi(s), & \hspace{-2em} D < |p| < H \text{  or  } \#[p \not= \pi] = K  \\
\emptyset, & \hspace{-2em} |p|=H  \\
\end{array}
\right.
\end{flalign*}
\end{small}

All members of the LDCF family are $\pi$-consistent by construction. However, monotonicity of an LDCF requires constraining $\Delta$. We say that $\Delta$ is \emph{depth monotonic} if $\Delta(s,d)\supseteq \Delta(s,d+1)$ for all $s\in S$ and $d$.\\
\begin{theorem}
\label{thm:LDCFmonotonicity}
For LDCF parameters $\theta=(\pi, H, K, D, \Delta)$, if $\Delta$ is depth monotonic, then $\psi_{\theta}$ is monotonic.
\end{theorem}

\noindent The proof is in the appendix. A straightforward way to obtain a depth monotonic $\Delta$ is to use a learned action-ranking function over states and return the top ranked actions at a state, where the number of returned actions is non-increasing with tree depth.\\

{\bf Application to Special Cases.} The choice function $\psi_{lds}$ 
is a restricted LDCF with $\Delta(s,d)=A$, which trivially satisfies our safety conditions. The LDCF space provides more flexibility on how to better control the introduction of discrepancies compared to traditional LDS.

The policy rollout choice function $\psi_{ro}$ is a special case of an LDCF with $D=0$ and $\Delta(s,0)=A$, which allows all choices at the root and only the base policy's choices thereafter. Since $\Delta$ is trivially depth monotonic, our safety result applies. This can be generalized to multi-step look-ahead rollout \cite{Bertsekas:1996:NP:560669}, where the top $M$ levels of the tree are fully expanded followed by restricting actions to those of the base policy until the horizon. Specifically, $D=M$ and $\Delta(s,d)=A$, which again satisfies our safety conditions. Note that when $D=H-1$ this degenerates to value iteration with horizon $H$.

Finally, the pruned-search choice function $\psi_{q,k}$ is an LDCF with a specific choice of discrepancy function $\text{TOP}_{q,k}$. Our safety conditions specify sufficient constraints that the action ranking function $q$ should satisfy. When $q$ is history independent, $\text{TOP}_{q,k}$ is depth monotonic. Otherwise, if $q$ is history-dependent, e.g. a recurrent neural network, no such guarantee can be made. However, it is relatively straightforward to put a wrapper around such a $q$ to ensure depth monotonicity.

{\bf Computational Complexity.} Increasing $H$, $K$, and $D$, and the size of sets returned by $\Delta$ can be expected to improve $\Pi^{\psi_{\theta}}_u$ for reasonable $u$. This comes at the cost of higher computational complexity, which is typically dominated by the number of leaves in $T^{\psi_{\theta}}$. 
In addition to the LDCF parameters, the number of leaf nodes depends on the \emph{state branching factor} $C$, which is the maximum number of state nodes under an action node. When an exact model is used, this is the maximum number of non-zero probability successor states. For Monte-Carlo algorithms, this is the number of successor states sampled for each action node. Given the state branching factor and an upper bound on the number of actions returned by $\Delta$ we get the following bound on the tree size.
\begin{proposition}
Let $\psi_{\theta}$ be an LDCF with $\theta=(\pi, H, K, D, \Delta)$, such that $\Delta(s)\leq W$ for any $s\in S$. The number of leaf nodes in $T^{\psi_{\theta}}$ with state branching factor $C$ is upper bounded by $2C^H$ for $(D+1)W=1$ and otherwise by $\frac{\left((D+1)W\right)^{K+1}-1}{(D+1)W-1}C^H = O\left((DW)^KC^H\right)$.
\end{proposition}

\noindent Ignoring the impact of $C$, which is controlled by the search algorithm, the complexity is dominated by $K$. We also see that for deterministic domains where $C=1$, there is no exponential dependence on $H$.

\section{Illustrative Empirical Results} 
\label{sec:experiments}

\begin{figure*}[t]
\begin{center}
    \includegraphics[width=0.45\textwidth, height=0.21\textheight]{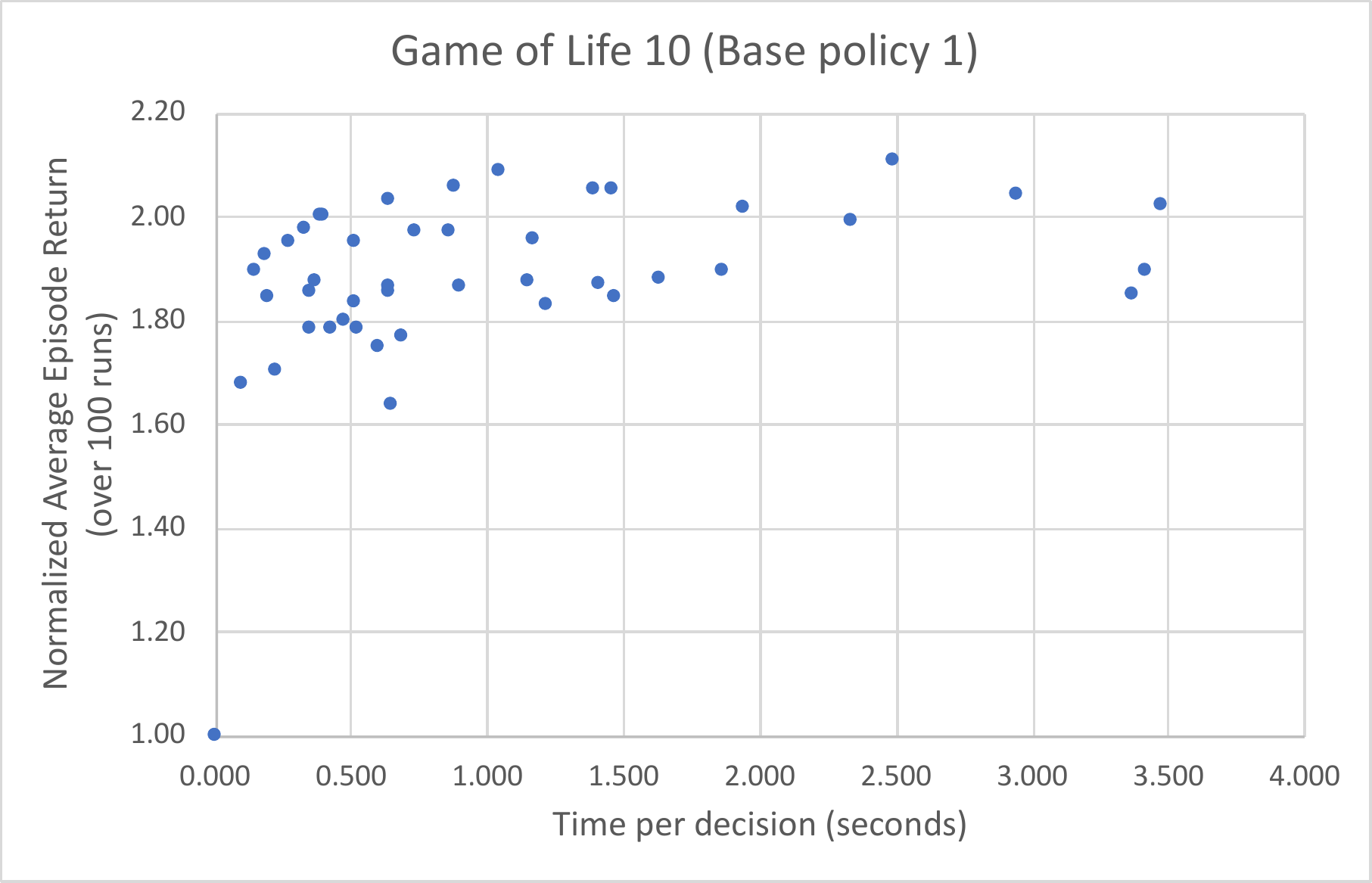}
    \includegraphics[width=0.45\textwidth, height=0.21\textheight]{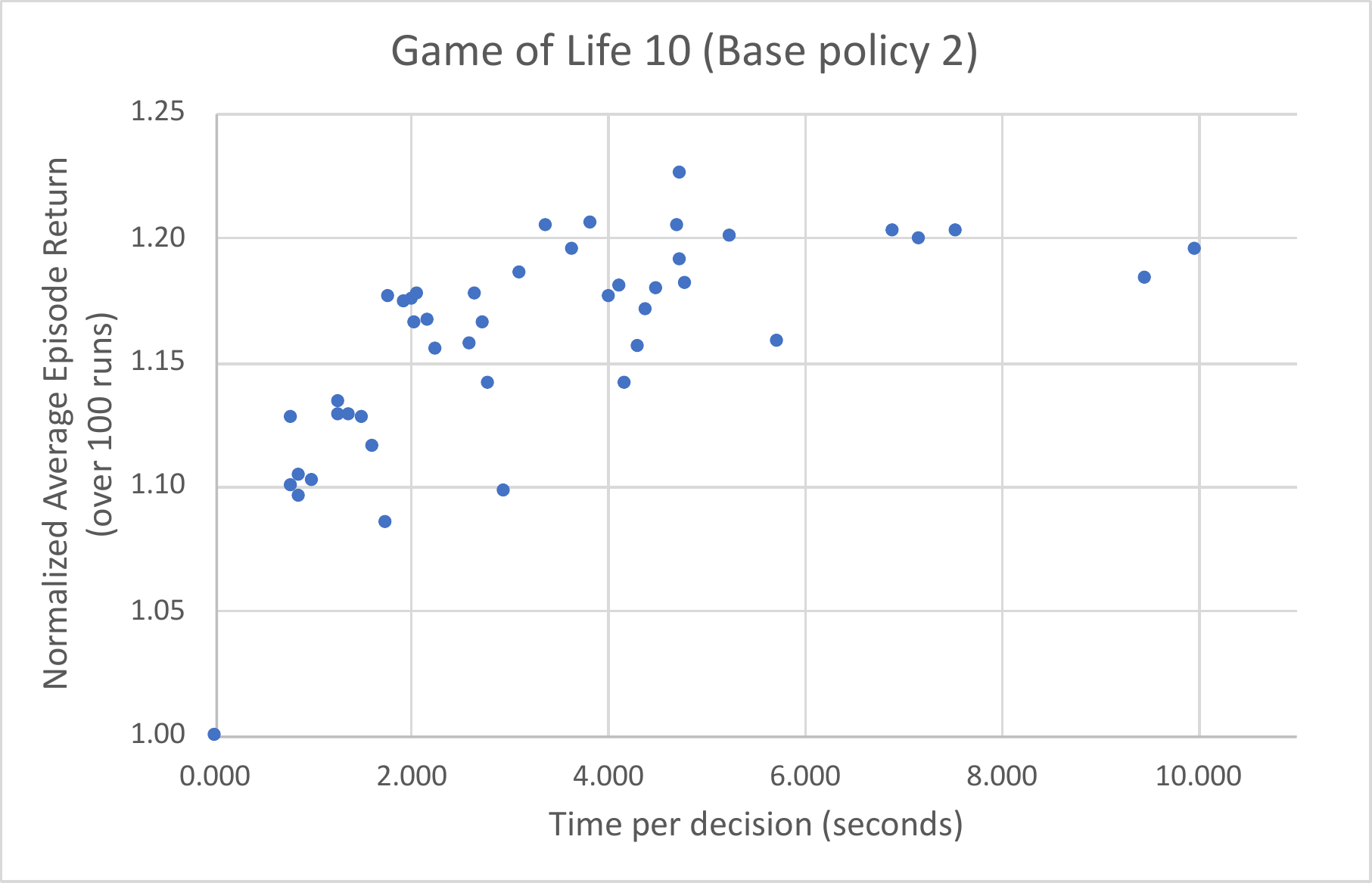}
\caption{Scatter plot of performance versus time-per-step for different LCDF choice functions applied to linear (left) and non-linear (right) base policies on the largest problem from Game-of-Life.}
\label{fig:2}
\vspace{-1em}
\end{center}
\end{figure*}

Our primary contribution is theoretical, however, here we illustrate the potential practical utility of our framework using the LDCF family. The experiments are intended to illustrate the viewpoint of a choice function being a hyperparameter to be tuned offline per application.  

We implemented a variant of Forward Search Sparse Sampling (FSSS) \cite{Walsh:2010:ISP:2898607.2898706} for approximately computing the online policy $\Pi^{\psi}_u$ for any LDCF $\psi$ and leaf-evaluation function $u$ using an MDP simulator. The key parameter, other than the choice-function, is the sampling width $C$, which controls how many state transitions are sampled for each action node. The supplementary material contains a summary of the algorithm. Our implementation is generally applicable and will be released as open source upon publication, along with information to reproduce our experiments. \\

{\bf Experiments Setup.} A full image of our experimental environment will be available upon publication. We run experiments in the domain Game-of-Life, a benchmark domain from the International Probabilistic Planning Competition. 
This is a grid-based domain with each grid-cell either being alive with some probability depending on its neighbors.
Actions allow for selecting one cell at each time step (or none) to set to be alive in the next  time step. The reward is based on the number of alive cells at each step. There are 10 problems of grid sizes $3 \times 3$, $4 \times 4$, $5 \times 5$ and $10 \times 3$. In addition different problems have different levels of stochasticity in the dynamics. 

We used supervised learning via imitation of a planner to train two base policies represented as neural networks, using the same approach as in \cite{issakkimuthu2018training}. Each network outputs a probability distribution over actions. Policy 1 is a linear network, while Policy 2 is a non-linear network with 3 hidden layers. For each base policy, we consider four leaf evaluation functions. The first is the constant zero function. The remaining three are neural networks with different configurations trained on a dataset of $5000$ state-value pairs obtained via Monte-Carlo simulation of each policy. All the networks have been trained using Tensorflow to minimize the mean squared error.

We experiment with $11$ choice functions in the LDCF family with parameters $H \in \{3,4,5\}$, $(D,K) \in \{(0,1), (1,1), (1,2), (2,1)\}$. The combination $(2,1)$ is not applicable for $H=3$. The number of discrepancies considered at the root node is $9$ and other internal nodes is $1$. 
The discrepancies are determined from the action probabilities given by the base policy. We use $C=3$ for FSSS.

Given one of these policies and a problem, we would like to identify the best combination of LDCF parameters and leaf evaluation function given a constraint on the time-per-step. In practice, this could be done in an offline tuning phase where different configurations are evaluated. Figure \ref{fig:2} shows a scatter plot of the normalized reward versus time-per-step for each of the 44 configurations (11 LDCF settings and 4 leaf evaluation functions). Here the normalized reward is the average reward per episode divided by the average reward per episode of the base policy. Values greater (less) than one perform better (worse) than the base policy. For both base policies all LDCF configurations perform better. There is also a larger improvement for base policy 1, which makes sense due to the fact that policy 2 is a much higher-quality policy and hence more difficult to improve over. We also see that the LDCF space shows considerable coverage of the performance vs. decision-time space, which shows the utility of offline tuning. There is a general trend toward better performance for larger decision times, but this is not uniformly true. There are complex interactions between the LDCF parameters and a particular problem, which makes it unlikely that a single feature such as time-per-decision is always the best indicator. \\

\begin{tablehere}
\begin{footnotesize}
\begin{center}
\centering
\begin{tabular}{l|rr|rr}
\hline
\hline
& \multicolumn{2}{c|}{\bf LDCF Policy 1} & \multicolumn{2}{c}{\bf LDCF Policy 2} \\
\hline
Prob. \# & \multicolumn{1}{c}{Normalized} & \multicolumn{1}{c|}{Decision}  & \multicolumn{1}{c}{Normalized} & \multicolumn{1}{c}{Decision}  \\
& \multicolumn{1}{c}{Avg. Reward} & \multicolumn{1}{c|}{Time (s)} & \multicolumn{1}{c}{Reward} & \multicolumn{1}{c}{Time (s)} \\
\hline
1 & 2.57 $\pm$ 0.07  & 0.081 & 1.08 $\pm$ 0.04 & 0.491 \\
\hline
2 & 1.27 $\pm$ 0.10  & 0.345 & 0.95 $\pm$ 0.06 & 0.631 \\
\hline
3 & 1.11 $\pm$ 0.05  & 0.129 & 0.92 $\pm$ 0.03 & 0.374 \\
\hline
4 & 1.51 $\pm$ 0.03  & 0.523 & 1.03 $\pm$ 0.02 & 0.830 \\
\hline
5 & 1.14 $\pm$ 0.03  & 1.084 & 1.00 $\pm$ 0.03 & 6.298 \\
\hline
6 & 1.05 $\pm$ 0.02  & 1.223 & 0.96 $\pm$ 0.02 & 9.163 \\
\hline
7 & 1.54 $\pm$ 0.02  & 0.523 & 1.05 $\pm$ 0.01 & 1.957 \\
\hline
8 & 1.21 $\pm$ 0.02  & 4.488 & 1.02 $\pm$ 0.02 & 10.463 \\
\hline
9 & 1.13 $\pm$ 0.02  & 3.262 & 0.96 $\pm$ 0.01 & 3.749 \\
\hline
10 & 2.11 $\pm$ 0.04  & 2.493 & 1.23 $\pm$ 0.02 & 4.746 \\
\hline \hline
\end{tabular}
\caption{Game-of-Life. Normalized reward of best performing LDCF configuration for each of the base policies. }
\label{tab:linear}
\end{center}
\end{footnotesize}
\end{tablehere}

Table \ref{tab:linear} gives results for each of the 10 problems, which includes the normalized rewards with confidence intervals for the best performing LDCF configuration for each of the policies. We see that for the linear policy, the best LDCF configuration is never significantly worse (lower interval is greater than 1) and often significantly better. For the second non-linear policy, we see that for most problems the LDCF performance is not significantly worse than the policy (confidence interval contains 1) and sometimes significantly better. For three problems, the upper confidence bound is less than one, indicating a significant decrease in performance. These problems happen to be among the most stochastic problems in the benchmark set. This suggests that a likely reason for the decrease in performance is due to the relatively small sampling width used for FSSS ($C=3$), which provides a poor approximation for such problems. 

\section{Summary}
We have introduced a framework for analyzing online search procedures for policy improvement guarantees. The key idea is to separate the action specification part of search from the search process and create an abstract concept called choice functions. A choice function instance will then be a parameter of search. We identify properties of choice functions to provide sufficient conditions for guaranteed online policy improvement when  the leaf evaluation function is perfect. Our main result is a bound on the performance of the online policy relative to the base policy for any leaf evaluation function. We have also introduced a parameterized class of choice functions called LDCF and discussed online search with FSSS for choice functions in this class. Our next directions are to thoroughly explore the practical application of the framework across a wide range of problems and to integrate notions of state abstraction into the framework.  

\section{Acknowledgements}
This work was supported by NSF grant IIS-1619433 and DARPA contract N66001-17-2-4030. We thank Intel for assisting with compute support for this work.

\bibliography{CF}

\begin{thebibliography}{10}

\bibitem{bertsekas1999scheduling}
D.~P. Bertsekas and D.~A. Castanon.
\newblock Rollout {A}lgorithms for {S}tochastic {S}cheduling {P}roblems.
\newblock {\em Journal of Heuristics}, 5(1):89--108, 1999.

\bibitem{Bertsekas:1996:NP:560669}
D.~P. Bertsekas and J.~N. Tsitsiklis.
\newblock {\em Neuro-Dynamic Programming}.
\newblock Athena Scientific, 1996.

\bibitem{bertsekas1997combinatorial}
D.~P. Bertsekas, J.~N. Tsitsiklis, and C.~Wu.
\newblock Rollout {A}lgorithms for {C}ombinatorial {O}ptimization.
\newblock {\em Journal of Heuristics}, 3(3):245--262, 1997.

\bibitem{browne2012survey}
C.~B. Browne, E.~Powley, D.~Whitehouse, S.~M. Lucas, P.~I. Cowling,
  P.~Rohlfshagen, S.~Tavener, D.~Perez, S.~Samothrakis, and S.~Colton.
\newblock A {S}urvey of {M}onte {C}arlo {T}ree {S}earch {M}ethods.
\newblock {\em IEEE Transactions on Computational Intelligence and AI in
  games}, 4(1):1--43, 2012.

\bibitem{cazenave2009nested}
T.~Cazenave.
\newblock Nested {M}onte-{C}arlo {S}earch.
\newblock In {\em Twenty-First International Joint Conference on Artificial
  Intelligence}, 2009.

\bibitem{Chang:2004:PRO:990742.990814}
H.~S. Chang, R.~Givan, and E.~K.~P. Chong.
\newblock Parallel {R}ollout for {O}nline {S}olution of {P}artially
  {O}bservable {M}arkov {D}ecision {P}rocesses.
\newblock {\em Discrete Event Dynamic Systems}, 14(3):309--341, July 2004.

\bibitem{doppa2014structured}
J.~R. Doppa, A.~Fern, and P.~Tadepalli.
\newblock Structured {P}rediction via {O}utput {S}pace {S}earch.
\newblock {\em The Journal of Machine Learning Research}, 15(1):1317--1350,
  2014.

\bibitem{harvey1995limited}
W.~D. Harvey and M.~L. Ginsberg.
\newblock Limited {D}iscrepancy {S}earch.
\newblock In {\em IJCAI (1)}, pages 607--615, 1995.

\bibitem{issakkimuthu2018training}
M.~Issakkimuthu, A.~Fern, and P.~Tadepalli.
\newblock Training deep reactive policies for probabilistic planning problems.
\newblock In {\em Twenty-Eighth International Conference on Automated Planning
  and Scheduling}, 2018.

\bibitem{Kearns:2002:SSA:599616.599698}
M.~Kearns, Y.~Mansour, and A.~Y. Ng.
\newblock A {S}parse {S}ampling {A}lgorithm for near-optimal {P}lanning in
  {L}arge {M}arkov {D}ecision {P}rocesses.
\newblock {\em Machine Learning}, 49(2-3):193--208, Nov. 2002.

\bibitem{larrosa2016limited}
F.~J. Larrosa~Bondia, E.~Roll{\'o}n~Rico, and R.~Dechter.
\newblock Limited {D}iscrepancy {AND/OR} {S}earch and its {A}pplication to
  {O}ptimization {T}asks in {G}raphical {M}odels.
\newblock In {\em Proceedings of the Twenty-Fifth International Joint
  Conference on Artificial Intelligence: New York, New York, USA, 9--15 July
  2016}, pages 617--623, 2016.

\bibitem{nguyen2014bootstrapping}
T.-H.~D. Nguyen, T.~Silander, W.-S. Lee, and T.-Y. Leong.
\newblock Bootstrapping {S}imulation-based {A}lgorithms with a {S}uboptimal
  {P}olicy.
\newblock In {\em Twenty-Fourth International Conference on Automated Planning
  and Scheduling}, 2014.

\bibitem{pinto2017}
J.~Pinto and A.~Fern.
\newblock Learning {P}artial {P}olicies to {S}peedup {MDP} {T}ree {S}earch via
  {R}eduction to {IID} {L}earning.
\newblock {\em The Journal of Machine Learning Research}, 18(1):2179--2213,
  2017.

\bibitem{silver2016alphago}
D.~Silver, A.~Huang, C.~J. Maddison, A.~Guez, L.~Sifre, G.~Van Den~Driessche,
  J.~Schrittwieser, I.~Antonoglou, V.~Panneershelvam, M.~Lanctot, et~al.
\newblock Mastering the {G}ame of {G}o with {D}eep {N}eural {N}etworks and
  {T}ree {S}earch.
\newblock {\em Nature}, 529(7587):484, 2016.

\bibitem{Silver1140}
D.~Silver, T.~Hubert, J.~Schrittwieser, I.~Antonoglou, M.~Lai, A.~Guez,
  M.~Lanctot, L.~Sifre, D.~Kumaran, T.~Graepel, T.~Lillicrap, K.~Simonyan, and
  D.~Hassabis.
\newblock A {G}eneral {R}einforcement {L}earning {A}lgorithm that masters
  {C}hess, {S}hogi, and {G}o through self-play.
\newblock {\em Science}, 362(6419):1140--1144, 2018.

\bibitem{silver2017gozero}
D.~Silver, J.~Schrittwieser, K.~Simonyan, I.~Antonoglou, A.~Huang, A.~Guez,
  T.~Hubert, L.~Baker, M.~Lai, A.~Bolton, et~al.
\newblock Mastering the {G}ame of {G}o without {H}uman {K}nowledge.
\newblock {\em Nature}, 550(7676):354, 2017.

\bibitem{tesauro1997}
G.~Tesauro and G.~R. Galperin.
\newblock On-line {P}olicy {I}mprovement using {M}onte-{C}arlo {S}earch.
\newblock In {\em Advances in Neural Information Processing Systems}, pages
  1068--1074, 1997.

\bibitem{Walsh:2010:ISP:2898607.2898706}
T.~J. Walsh, S.~Goschin, and M.~L. Littman.
\newblock Integrating sample-based planning and model-based reinforcement
  learning.
\newblock In {\em Proceedings of the Twenty-Fourth AAAI Conference on
  Artificial Intelligence}, AAAI'10, pages 612--617. AAAI Press, 2010.

\bibitem{yan2005solitaire}
X.~Yan, P.~Diaconis, P.~Rusmevichientong, and B.~V. Roy.
\newblock Solitaire: {M}an versus {M}achine.
\newblock In {\em Advances in Neural Information Processing Systems}, pages
  1553--1560, 2005.

\end{thebibliography}
\bibliographystyle{abbrv}

\clearpage
\section*{Appendix}
\subsection*{Omitted Proofs}
\setcounter{lemma}{0}
\begin{lemma}
For any policy $\pi$ and value vector $V$, if \\ $V - B_{\pi}[V] \leq \delta$, then $V - V^{\pi} \leq \dfrac{\delta}{1 - \gamma}$. 
\end{lemma}
\begin{proof}
Let $P^{\pi} \in \mathbb{R}^{n \times n}$, $R^{\pi} \in \mathbb{R}^n$ and $V^{\pi} \in \mathbb{R}^n$ be the state-transition matrix, reward and value functions of $\pi$ respectively. We use the matrix notation $V^\pi = R^\pi + \gamma P^\pi V^\pi$ to mean
\begin{flalign*}
V^\pi(s) = R(s, \pi(s)) + \gamma \displaystyle \sum_{s' \in S} P_{ss'}(\pi(s)) \cdot V^\pi(s')
\end{flalign*}
for all $s \in S$. If $D = V - V^{\pi}$ then
\begin{flalign*}
D &=  V - R^{\pi} - \gamma P^{\pi} V^{\pi}, \hspace{2.2em} \text{\emph{ since }} V^{\pi} = R^{\pi} + \gamma P^{\pi} V^{\pi} \hspace{1em} \\
&= V - R^{\pi} - \gamma P^{\pi} V^{\pi} + \gamma P^{\pi} V - \gamma P^{\pi} V, \\
& \hspace{12.9em} \text{\emph{add and subtract }} \gamma P^{\pi} V \\
&= \gamma P^{\pi} (V - V^{\pi}) + V - R^{\pi} - \gamma P^{\pi} V \\
&=  \gamma P^{\pi} D + (V - B_{\pi}[V]), \\
& \hspace{11.2em} \text{\emph{since }} B_{\pi}[V] = R^{\pi} + \gamma P^{\pi} V \\
&= (V - B_{\pi}[V]) +  \gamma P^{\pi} D. \\
\end{flalign*}
Therefore, $D = (V - B_{\pi}[V]) +  \gamma P^{\pi} D$ is the fixed-point equation of a policy with $P^{\pi}$ as the state-transition matrix and $V - B_{\pi}[V]$ as the reward function. Since the maximum value of the value function of a policy with maximum reward $R_{\max}$ is $\dfrac{R_{\max}}{1 - \gamma}$, we get
\begin{flalign*}
V  - V^\pi = D \leq \dfrac{V - B_{\pi}[V]}{1 - \gamma} \leq \dfrac{\delta}{1 - \gamma}. \hspace{5em}
\end{flalign*}
\end{proof}

\begin{lemma}
If a stationary choice function $\psi$ is $\pi$-consistent and monotonic and $u = V^\pi$ then for any path $p;s$ such that $1 \leq |p;s| \leq H(\psi)$, 
\begin{flalign*}
V^\psi_u(\lrcorner p;s) \geq V^\psi_u(p;s). \hspace{12em}
\end{flalign*}
\end{lemma}
\begin{proof}
By definition,
\begin{small}
\begin{flalign*}
V_u^\psi(p;s) = \left \{
\begin{array}{lr}
V^\pi(s), & \hspace{-7.4em} \text{\emph{ if }} \psi(p;s) = \emptyset, \\
\displaystyle \max_{a \in  \psi(p;s)} R(s,a) + \gamma \displaystyle \sum_{s' \in S}  P_{ss'}(a) V_u^\psi(p;s;a;s'), \\
& \hspace{-7.4em} \text{\emph{ otherwise}}.\\
\end{array}
\right.
\end{flalign*}
\end{small}

\noindent The proof is by induction on $|p;s|$. If $|p;s| = H(\psi)$ then $\psi(p;s) = \emptyset$. \\

\noindent \emph{Induction Basis. } Let $p;s$ be such that $|p;s| = H(\psi)$. 

\noindent Case $(1). \text{ } \psi(\lrcorner p;s) = \emptyset$. Since $|p;s|=H$, we have $\psi(p;s) = \emptyset$ and
\begin{flalign*}
V_u^\psi(\lrcorner p;s) = V_u^\psi(p;s) = V^\pi(s). \hspace{6em}
\end{flalign*}

\noindent Case $(2). \text{ } \psi(\lrcorner p;s) \neq \emptyset$. 
\begin{small}
\begin{flalign*}
V_u^\psi(\lrcorner p;s) &= \displaystyle \max_{a \in  \psi(\lrcorner p;s)} R(s,a) + \gamma \displaystyle \sum_{s' \in S}  P_{ss'}(a) \cdot V_u^\psi(\lrcorner p;s;a;s') \\
&  = \displaystyle \max_{a \in  \psi(\lrcorner p;s)} R(s,a) + \gamma \displaystyle \sum_{s' \in S} P_{ss'}(a) \cdot V^\pi(s'), \\
& \hspace{13em} \text{\emph{since }} |\lrcorner p;s;a;s'| = H \\
& \geq R(s,\pi(s)) + \gamma \displaystyle \sum_{s' \in S} P_{ss'}(\pi(s)) \cdot V^\pi(s'), \\
& \hspace{13.9em} \text{\emph{since $\pi(s) \in \psi(p;s)$}} \\
& = V^\pi(s) = V_u^\psi(p;s).
\end{flalign*}
\end{small}

\noindent \emph{Induction Hypothesis. } Assume that  $V_u^\psi(\lrcorner p;s) \geq V_u^\psi(p;s)$ for $p;s$ such that $|p;s| = k+1$.\\

\noindent \emph{Inductive Proof. } Let $p;s$ be such that $|p;s| = k$. 

\noindent Case $(1). \text{ } \psi(\lrcorner p;s) = \emptyset$. Since $\psi$ is monotonic, $\psi(p;s) \subseteq \psi(\lrcorner p;s)$ and hence $\psi(p;s) = \emptyset$. Therefore, 
\begin{flalign*}
V_u^\psi(\lrcorner p;s) = V_u^\psi(p;s) = V^\pi(s). \hspace{4em}
\end{flalign*}

\noindent Case $(2). \text{ }\psi(\lrcorner p;s) \neq \emptyset$. 
\begin{small}
\begin{flalign*}
V_u^\psi(\lrcorner p;s) &= \displaystyle \max_{a \in  \psi(\lrcorner p;s)} R(s,a) + \gamma \displaystyle \sum_{s' \in S} P_{ss'}(a) \cdot V_u^\psi(\lrcorner p;s;a;s') \\
& \geq \displaystyle \max_{a \in  \psi(\lrcorner p;s)} R(s,a) + \gamma \displaystyle \sum_{s' \in S} P_{ss'}(a) \cdot V_u^\psi(p;s;a;s'), \\
& \hspace{2em} \text{\emph{by induction hypothesis, since }} |p;s;a;s'| = k+1 \\
& \geq \displaystyle \max_{a \in  \psi(p;s)} R(s,a) + \gamma \displaystyle \sum_{s' \in S} P_{ss'}(a) \cdot V_u^\psi(p;s;a;s'), \\
& \hspace{12em} \text{\emph{since }} \psi(p;s) \subseteq \psi(\lrcorner p;s) \\
& = V_u^\psi(p;s).
\end{flalign*}
\end{small}
\end{proof}

\begin{lemma}
\label{lemma:contraction2}
If $\psi$ is a stationary choice function and $\lVert u - u' \rVert_\infty \leq \epsilon$ then for any path $p;s$ with $|p;s| \leq h(\psi)$,
\begin{flalign*}
\left | V^\psi_u(p;s) - V^\psi_{u'}(p;s) \right | \leq \epsilon \gamma^{h(\psi) - |p;s|}. \hspace{3em}
\end{flalign*}
\end{lemma}
\begin{proof}
The proof is by induction on $|p;s|$ for $|p;s| \leq h(\psi)$.\\

\noindent \emph{Induction Basis. } The lemma holds for $|p;s| = h(\psi)$, since by lemma \ref{lemma:contractionsubresult},  
\begin{flalign*}
\left | V^\psi_u(p;s) - V^\psi_{u'}(p;s) \right | \leq \epsilon \hspace{6em}
\end{flalign*}
for all $p;s$ with $h(\psi) \leq |p;s| \leq H(\psi)$.\\

\noindent \emph{Induction Hypothesis. } Assume that the lemma holds for $|p;s| = h(\psi) - (k - 1)$.  \\

\noindent \emph{Inductive Proof. } Let $p;s$ be a path such that $|p;s| = h(\psi) - k$. 
By the definition of $h(\psi)$, for any path $p;s$ such that $|p;s| < h(\psi)$, we have $\psi(p;s) \neq \emptyset$. Therefore, 
\begin{small}
\begin{flalign*}
\left |V^\psi_u(p;s) - V^\psi_{u'}(p;s)\right | &= \bigg | \displaystyle \max_{a \in \psi(p;s)} \bigg \{ R(s, a) + \hspace{7.5em} \\
&  \hspace{-6em} \gamma \displaystyle \sum_{s' \in S} P_{ss'}(a) \cdot V^\psi_u(p;s;a;s') \bigg \} - \displaystyle \max_{a \in \psi(p;s)} \bigg \{R(s, a) + \\
&  \hspace{2.9em} \gamma \displaystyle \sum_{s' \in S} P_{ss'}(a) \cdot V^\psi_{u'}(p;s;a;s') \bigg \} \bigg | \\
& \hspace{-8.5em} \leq \displaystyle \max_{a \in \psi(p;s)} \bigg |R(s, a) + \gamma \displaystyle \sum_{s' \in S} P_{ss'}(a) \cdot V^\psi_u(p;s;a;s') - \\
&  \hspace{-3.8em} \bigg \{ R(s, a) + \gamma \displaystyle \sum_{s' \in S} P_{ss'}(a) \cdot V^\psi_{u'}(p;s;a;s') \bigg \} \bigg | \\
& \hspace{-8.5em} = \displaystyle \max_{a \in \psi(p;s)} \bigg |\gamma \displaystyle \sum_{s' \in S} P_{ss'}(a) ( V^\psi_u(p;s;a;s') - V^\psi_{u'}(p;s;a;s')) \\
& \hspace{-8.5em} \leq \displaystyle \max_{a \in \psi(p;s)} \gamma \displaystyle \sum_{s' \in S} P_{ss'}(a) \text{ } \bigg | V^\psi_u(p;s;a;s') - V^\psi_{u'}(p;s;a;s') \bigg |\\
& \hspace{-8.5em} \leq \displaystyle \max_{a \in \psi(p;s)} \gamma \displaystyle \sum_{s' \in S} P_{ss'}(a) \text{ }  \epsilon \gamma^{h(\psi) - (k-1)}, \\
&  \hspace{5em}  \text{\emph{by the induction hypothesis}} \\
& \hspace{-8.5em} = \epsilon \gamma^{h(\psi) - k}.
\end{flalign*}
\end{small}
\end{proof}

\begin{lemma}
If a stationary choice function $\psi$ is $\pi$-consistent and monotonic  and $\lVert u - V^\pi \rVert_\infty \leq \epsilon_\pi$, then for $\pi' = \Pi^\psi_u$,
$V_{u, 0}^\psi - B_{\pi'}[V_{u, 0}^\psi] \leq \epsilon_\pi \gamma^{h(\psi)} (1 + \gamma)$. 
\end{lemma}
\begin{proof}
Let $\bar{u} = V^\pi$ to simplify notation.
\begin{small}
\begin{flalign*}
& V_{u, 0}^\psi(s) - B_{\pi'}[V_{u, 0}^\psi](s) = R(s, \pi'(s)) + \hspace{20em} \\
& \hspace{4em} \gamma \displaystyle \sum_{s' \in S} P_{ss'}(\pi'(s)) \cdot V_{u, 1}^\psi(s;\pi'(s);s') - \bigg \{ R(s, \pi'(s)) +  \\
& \hspace{14em} \gamma \displaystyle \sum_{s' \in S} P_{ss'}(\pi'(s)) \cdot V_{u, 0}^\psi(s') \bigg \} \\
& \hspace{1.5em} = \gamma \displaystyle \sum_{s' \in S} P_{ss'}(\pi'(s)) \cdot (V_{u, 1}^\psi(s;\pi'(s);s') - V_{u, 0}^\psi(s'))\\
& \hspace{1.5em} \leq \gamma \displaystyle \sum_{s' \in S} P_{ss'}(\pi'(s)) \bigg \{V_{\bar{u}, 1}^\psi(s;\pi'(s);s') + \epsilon_\pi \gamma^{h(\psi) - 1} \\
& \hspace{6em} - \bigg ( V_{\bar{u}, 0}^\psi(s') - \epsilon_\pi \gamma^{h(\psi)} \bigg ) \bigg \}, \hspace{2.3em} \text{\emph{ using lemma \ref{lemma:contraction}}}\\
& \hspace{1.5em} \leq \gamma \displaystyle \sum_{s' \in S} P_{ss'}(\pi'(s)) \cdot ( \epsilon_\pi \gamma^{h(\psi) - 1} + \epsilon_\pi \gamma^{h(\psi)}), \hspace{0.2em} \text{\emph{ by lemma \ref{lemma:monotonicity}}} \\
& \hspace{14.6em} V_{\bar{u}, 1}^\psi(s;\pi'(s);s') \leq V_{\bar{u}, 0}^\psi(s') \\
& \hspace{1.5em} = \epsilon_\pi \gamma^{h(\psi)} (1 + \gamma).
\end{flalign*}
\end{small}
\end{proof}

\begin{lemma}
\label{lemma:contractionsubresult}
If $\psi$ is a stationary choice function and $\lVert u - u' \rVert_\infty \leq \epsilon$ then for any path $p;s$ such that $h(\psi) \leq |p;s| \leq H(\psi)$, 
\begin{flalign*}
\left | V^\psi_u(p;s) - V^\psi_{u'}(p;s) \right | \leq \epsilon. \hspace{6em}
\end{flalign*}
\end{lemma}
\begin{proof}
The proof is by induction on $|p;s|$ for $h(\psi) \leq |p;s| \leq H(\psi)$.\\

\noindent \emph{Induction Basis. } Let  $p;s$ be a path such that $|p;s| = H(\psi)$. 
The lemma holds for $|p;s| = H(\psi)$, since $\psi(p;s) = \emptyset$ and
\begin{flalign*}
\left |V^\psi_u(p;s) - V^\psi_{u'}(p;s)\right | &= |u(s) - u'(s)| = \epsilon. \hspace{3em} 
\end{flalign*}

\noindent \emph{Induction Hypothesis. } Assume that the lemma holds for any path $p;s$ such that $h(\psi) < |p;s| = H(\psi) - (k - 1)$. \\

\noindent \emph{Inductive Proof. } \\
Let $p;s$ be a path such that $|p;s| = H(\psi) - k$. \\
\noindent Case $(1). $ If $\psi(p;s) = \emptyset$ then 
\begin{flalign*}
\left | V^\psi_u(p;s) - V^\psi_{u'}(p;s) \right | \leq \epsilon. \hspace{6em}\\
\end{flalign*}
\noindent Case $(2). $ If $\psi(p;s) \neq \emptyset$ then 
\begin{small}
\begin{flalign*}
\left |V^\psi_u(p;s) - V^\psi_{u'}(p;s)\right | &= \bigg | \displaystyle \max_{a \in \psi(p;s)} \bigg \{ R(s, a) + \hspace{7em} \\
&  \hspace{-6em} \gamma \displaystyle \sum_{s' \in S} P_{ss'}(a) \cdot V^\psi_u(p;s;a;s') \bigg \} - \displaystyle \max_{a \in \psi(p;s)} \bigg \{R(s, a) + \\
&  \hspace{2.9em} \gamma \displaystyle \sum_{s' \in S} P_{ss'}(a) \cdot V^\psi_{u'}(p;s;a;s') \bigg \} \bigg | \\
& \hspace{-8.5em} \leq \displaystyle \max_{a \in \psi(p;s)} \bigg |R(s, a) + \gamma \displaystyle \sum_{s' \in S} P_{ss'}(a) \cdot V^\psi_u(p;s;a;s') - \\
&  \hspace{-3.8em} \bigg \{ R(s, a) + \gamma \displaystyle \sum_{s' \in S} P_{ss'}(a) \cdot V^\psi_{u'}(p;s;a;s') \bigg \} \bigg | \\
& \hspace{-8.5em} = \displaystyle \max_{a \in \psi(p;s)} \bigg |\gamma \displaystyle \sum_{s' \in S} P_{ss'}(a) ( V^\psi_u(p;s;a;s') - V^\psi_{u'}(p;s;a;s') ) \bigg | \\
& \hspace{-8.5em} \leq \displaystyle \max_{a \in \psi(p;s)} \gamma \displaystyle \sum_{s' \in S} P_{ss'}(a) \text{ } \bigg | V^\psi_u(p;s;a;s') - V^\psi_{u'}(p;s;a;s') \bigg | \\
& \hspace{-8.5em} \leq \displaystyle \max_{a \in \psi(p;s)} \gamma \displaystyle \sum_{s' \in S} P_{ss'}(a) \cdot  \epsilon, \hspace{2em} \text{\emph{by the induction hypothesis}} \\
& \hspace{-8.5em} \leq \epsilon.
\end{flalign*}
\end{small}
\end{proof}

\setcounter{proposition}{0}
\begin{proposition}
If a stationary choice function $\psi$ is $\pi$-consistent and monotonic and $u = V^\pi$, then $V^\psi_{u}(s) \geq V^\pi(s)$.
\end{proposition}
\begin{proof}
Let $p;s$ be a path such that $|p;s| = k$ and $\psi(p;s) = \emptyset$. By definition, $V^\psi_u(p;s) = V^\pi(s)$. Let $\lrcorner^i p;s$ denote the path obtained from $p;s$ by removing the first $i$ state-action pairs of $p;s$. 
\begin{small}
\begin{flalign*}
V^\pi(s) &= V^\psi_u(p;s) \hspace{18em} \\
& \leq V^\psi_u(\lrcorner^i p;s) \leq V^\psi_u(\lrcorner^{i+1} p;s), \hspace{0.5em} 1 \leq i < k, \text{\emph{ by lemma \ref{lemma:monotonicity}}}\\
& = V^\psi_u(s)
\end{flalign*}
\end{small}
\end{proof}

\begin{lemma}
\label{lemma:subsume}
For stationary choice functions $\psi$ and $\psi'$, if $\psi$ and $\psi'$ have the same set of leaf paths and $\psi \supseteq \psi'$,
then for any path $p;s$ such that $|p;s| \leq H(\psi')$ and leaf evaluation function $u$, $V^\psi_u(p;s) \geq V^{\psi'}_u(p;s)$.
\end{lemma}
\begin{proof}
The proof is  by  induction on $H(\psi')$. Since $\psi \supseteq \psi'$ and $\psi'(p;s) = \emptyset \Rightarrow \psi(p;s) = \emptyset$, we have $H(\psi') = H(\psi)$.\\

\noindent \emph{Induction Basis. } Let $p;s$ be a path such that $|p;s| = H(\psi')$. Since $\psi(p;s) = \psi'(p;s) = \emptyset$,
\begin{flalign*}
V^{\psi}_u(p;s) = V^{\psi'}_u(p;s) = u(s).  \hspace{6em}
\end{flalign*}
The lemma holds for $|p;s| = H(\psi')$.\\

\noindent \emph{Induction Hypothesis. } Assume that the lemma holds for $|p;s| = k+1 < H(\psi')$.\\

\noindent \emph{Inductive Proof. } Let $p;s$ be a path with $|p;s| = k$. \\ 
\noindent Case $(1).$ If $\psi(p;s) = \emptyset$ then 
\begin{flalign*}
V^\psi_u(p;s) = V^{\psi'}_u(p;s) = u(s). \hspace{8em}\\
\end{flalign*}
\noindent Case $(2).$ If $\psi(p;s) \neq \emptyset$ then
\begin{small}
\begin{flalign*}
V^\psi_u(p;s) &= \displaystyle \max_{a \in \psi(p;s)} R(s, a) + \gamma \displaystyle \sum_{s' \in S} P_{ss'}(a) \cdot V^\psi_u(p;s;a;s') \hspace{2em} \\
& \leq \displaystyle \max_{a \in \psi(p;s)} R(s, a) + \gamma \displaystyle \sum_{s' \in S} P_{ss'}(a) \cdot V^{\psi'}_u(p;s;a;s'), \\
& \hspace{11em} \text{\emph{by the induction hypothesis}} \\
& \leq \displaystyle \max_{a \in \psi'(p;s)} R(s, a) + \gamma \displaystyle \sum_{s' \in S} P_{ss'}(a) \cdot V^{\psi'}_u(p;s;a;s'), \\
& \hspace{12.5em} \text{\emph{since }} \psi(p;s) \supseteq \psi'(p;s)\\
& = V^{\psi'}_u(p;s).
\end{flalign*}
\end{small}
\end{proof}

\setcounter{theorem}{1}
\begin{theorem}
Let $\Psi = (\psi_1, \psi_2, \hdots)$ be a non-stationary choice function such that each component choice function $\psi_t$ is monotonic and $\pi$-consistent and $\lVert u - V^\pi \rVert_\infty = \epsilon_\pi$. If all $\psi_t$ have the same set of leaf paths and for each time step $t$, $\psi_{t+1} \supseteq \psi_t$, then for $\pi' = \Pi^\psi_u$,
\begin{flalign*}
V^\pi(s) - V^{\pi'}(s) \leq \dfrac{2 \epsilon_\pi \gamma^{h(\psi_1)}}{1 - \gamma}. \hspace{6em}
\end{flalign*}
 for all $s \in S$.
\end{theorem}
\begin{proof}
Let $\pi'_t = \Pi^{\psi_t}_u$ denote the stationary online policy for the component choice function $\psi_t$ and leaf evaluation function $u$. Let $P^{\pi'_t}$, $R^{\pi'_t}$ and $V^{\pi'_t}$ be the state-transition matrix, the reward and value functions of policy $\pi'_t$. From lemma \ref{lemma:qadvantage}, 
\begin{flalign*}
V_{u, 0}^{\psi_t} - B_{\pi'_t}[V_{u, 0}^{\psi_t}] \leq \epsilon_\pi \gamma^{h(\psi_t)} (1 + \gamma). \hspace{2em}
\end{flalign*}
Since $\psi_{t+1} \supseteq \psi_t$ and $\psi_1(p;s) = \emptyset \Rightarrow \psi_t(p;s) = \emptyset$ for $t \in \{1, 2, \hdots \}$, the min-horizon $h(\psi_t)$ is the same for all the component stationary choice functions $\psi_t$. Therefore,
\begin{flalign*}
\hspace{1em} &V_{u, 0}^{\psi_t} - B_{\pi'_t}[V_{u, 0}^{\psi_t}] \leq \epsilon_\pi \gamma^{h(\psi_1)} (1 + \gamma) \hspace{11em} \\
\Rightarrow & V_{u, 0}^{\psi_t} - R^{\pi'_t} - \gamma P^{\pi'_t} V_{u, 0}^{\psi_t} \leq \epsilon_\pi \gamma^{h(\psi_1)} (1 + \gamma), \\
& \hspace{4em} \text{\emph{since }} B_\pi[V] = R^\pi + \gamma P^\pi V \text{\emph{ for any policy }} \pi \\
\Rightarrow & R^{\pi'_t} \geq V_{u, 0}^{\psi_t} - \gamma P^{\pi'_t} V_{u, 0}^{\psi_t} - \epsilon_\pi \gamma^{h(\psi_1)} (1 + \gamma).
\end{flalign*}
Since $V_{u, 0}^{\psi_{t+1}} \geq V_{u, 0}^{\psi_t}$ \text{\emph{ by lemma \ref{lemma:subsume}}},
\begin{flalign*}
R^{\pi'_t} \geq - \epsilon_\pi \gamma^{h(\psi_1)} (1 + \gamma) + V_{u, 0}^{\psi_t} - \gamma P^{\pi'_t} V_{u, 0}^{\psi_{t+1}}. \hspace{3em} \\
\end{flalign*}
Let $\bar{u} = V^\pi$ to simplify notation. The value function of the online non-stationary policy is
\begin{small}
\begin{flalign*}
V^{\pi'} &= \displaystyle \sum_{t=1}^\infty \bigg [ \prod_{k=1}^{t-1} \gamma P^{\pi'_k} \bigg ] R^{\pi'_t} \hspace{17em} \\
& \hspace{-1.5em} \geq \displaystyle \sum_{t=1}^\infty \bigg [ \prod_{k=1}^{t-1} \gamma P^{\pi'_k} \bigg ] [- \epsilon_\pi \gamma^{h(\psi_1)} (1 + \gamma) + V_{u, 0}^{\psi_t} -  \gamma P^{\pi'_t} V_{u, 0}^{\psi_{t+1}} ] \\
& \hspace{-0.5em}  = \dfrac{- \epsilon_\pi \gamma^{h(\psi_1)} (1 + \gamma)}{1 - \gamma} + \displaystyle \sum_{t=1}^\infty \bigg [ \prod_{k=1}^{t-1} \gamma P^{\pi'_k} \bigg ] V_{u, 0}^{\psi_t} - \\
& \hspace{9em} \displaystyle \sum_{t=1}^\infty \bigg [ \prod_{k=1}^{t-1} \gamma P^{\pi'_k} \bigg ] ( \gamma P^{\pi'_t} V_{u, 0}^{\psi_{t+1}})\\
& \hspace{-0.5em} = \dfrac{- \epsilon_\pi \gamma^{h(\psi_1)} (1 + \gamma)}{1 - \gamma} + V_{u, 0}^{\psi_1} + \displaystyle \sum_{t=2}^\infty \bigg [ \prod_{k=1}^{t-1} \gamma P^{\pi'_k} \bigg ] V_{u, 0}^{\psi_t}  -  \\
& \hspace{9em} \displaystyle \sum_{t=1}^\infty \bigg [ \prod_{k=1}^{t} \gamma P^{\pi'_k} \bigg ] V_{u, 0}^{\psi_{t+1}}
\end{flalign*}
\end{small}
\begin{small}
\begin{flalign*}
\hspace{1em} & = \dfrac{- \epsilon_\pi \gamma^{h(\psi_1)} (1 + \gamma)}{1 - \gamma} + V_{u, 0}^{\psi_1}, \text{\emph{ cancelling out terms $2$ and $4$}} \\
& \geq \dfrac{- \epsilon_\pi \gamma^{h(\psi_1)} (1 + \gamma)}{1 - \gamma} + V_{\bar{u}, 0}^{\psi_1} - \epsilon_\pi \gamma^{h(\psi_1)}, \hspace{1.7em} \text{\emph{ by lemma \ref{lemma:contraction}}} \\
&  \geq \dfrac{- \epsilon_\pi \gamma^{h(\psi_1)} (1 + \gamma)}{1 - \gamma} + V^\pi - \epsilon_\pi \gamma^{h(\psi_1)}, \text{\emph{ since }} V_{\bar{u}, 0}^{\psi_1} \geq V^\pi \\
&  \hspace{17.8em} \text{\emph{by proposition \ref{prop:chaining}}} \\
& = \dfrac{- 2 \epsilon_\pi \gamma^{h(\psi_1)}}{1 - \gamma} + V^\pi. \\
\end{flalign*}
\end{small}
Therefore, we get
\begin{flalign*}
V^\pi - V^{\pi'} \leq \dfrac{2 \epsilon_\pi \gamma^{h(\psi_1)}}{1 - \gamma}. \hspace{9em}
\end{flalign*}
\end{proof}

\begin{theorem}
For LDCF parameters $\theta=(\pi, H, K, D, \Delta)$, if $\Delta$ is depth monotonic, then $\psi_{\theta}$ is monotonic.
\end{theorem}
\begin{proof}
Consider any path $p;s$ such that $1\leq |p;s| \leq H$ and let $\lrcorner p;s = p';s$. We must show that $\psi_{\theta}(p;s)\subseteq \psi_{\theta}(p';s)$. The definition of $\psi_{\theta}$ has three cases, which condition on the input path's length and number of discrepancies. $p';s$ never increases those quantities compared to $p;s$, since $|p'|=|p|-1$ and $p'$ does not contain state-action pairs that are not in $p$ and hence cannot increase the number of discrepancies. This means that there are two cases to consider: 1) $p';s$ satisfies a condition higher in the order than $p;s$, or 2) $p';s$ satisfies the same condition as $p;s$. For (1), $\psi_{\theta}(p;s)\subseteq \psi_{\theta}(p';s)$ is clearly satisfied when moving either from the bottom to middle or middle to top condition. For (2) we have $\psi_{\theta}(p;s) = \psi_{\theta}(p';s)$ for the bottom and middle conditions. The top condition gives $\psi_{\theta}(p;s)\subseteq \psi_{\theta}(p';s)$ due to the fact that $\Delta$ is depth monotonic. \\
\end{proof}

\setcounter{proposition}{2}
\begin{proposition}
Let $\psi_{\theta}$ be an LDCF with $\theta=(\pi, H, K, D, \Delta)$, such that $\Delta(s)\leq W$ for any $s\in S$. The number of leaf nodes in $T^{\psi_{\theta}}$ with state branching factor $C$ is upper bounded by $2C^H$ for $(D+1)W=1$ and otherwise by $\frac{\left((D+1)W\right)^{K+1}-1}{(D+1)W-1}C^H = O\left((DW)^KC^H\right)$.
\end{proposition}
\begin{proof}
The number of leaf nodes is bounded by the number of root-to-leaf paths in $T^{\psi_{\theta}}$. Each path is a sequence of alternating state and action nodes containing $H+1$ states and $H$ actions. Each path has at most $K$ discrepancies, each discrepancy being the choice of one of the first $D+1$ action nodes and one of at most $W$ discrepancies returned by $\Delta$. Thus, the total number of combinations of $K$ discrepancies is bounded by $\left((D+1)W\right)^K$. This bounds the combinations of $K$ or fewer discrepancies by $\frac{\left((D+1)W\right)^{K+1}-1}{(D+1)W-1}$ if $(D+1)W > 1$ and by $2$ if $(D+1)W=1$. For each discrepancy combination the number of paths is bounded by the number of ways to assign values to the $H$ state nodes, which is no greater than $C^H$. Multiplying this with the number of discrepancy combinations completes the proof. \\
\end{proof}

\subsection*{Forward Search Sparse Sampling for Choice Functions}
FSSS \cite{Walsh:2010:ISP:2898607.2898706} is a computationally efficient implementation of Sparse Sampling (SS) \cite{Kearns:2002:SSA:599616.599698}. SS simply builds a search tree of depth $H$ rooted at the current state $s$, with exactly $C$ successor state-nodes for every action-node, using a simulator of the true MDP. The root action values are computed and used to select the best action. 
FSSS improves over SS by incrementally constructing the tree via trajectory rollouts as in MCTS algorithms. This allows for sound pruning of the tree by incrementally maintaining intervals $[L_d(s), U_d(s)]$ and $[L_d(s,a), U_d(s,a)]$ at each depth $d$ state-node and action-node respectively. The intervals are guaranteed to contain the state and action values at those nodes, which supports pruning. For example, an action node whose upper bound is lower than another action's lower bound can be safely ignored.  

Adapting FSSS to search over at tree $T^{\psi}$ for any search function is straightforward. Since FSSS expands the tree via root-to-leaf rollouts, FSSS simply is adapted to only expand actions allowed by the choice function along the path of each rollout. For the LDCF family our implementation has an additional efficiency enhancement. Two paths $p;s$ and $p';s$ with equal numbers of discrepancies have $\psi_{\theta}(p;s)=\psi_{\theta}(p';s)$ and thus can be merged to get a Directed Acyclic Graph (DAG) rather than a tree. The DAGs can be significantly more compact than the corresponding trees, which leads to substantially faster search. 
\end{document}